\RequirePackage{fix-cm}
\documentclass[smallextended]{svjour3}       
\smartqed  
\usepackage{graphicx}

\usepackage[utf8]{inputenc} 
\usepackage[T1]{fontenc}    
\usepackage{hyperref}       
\usepackage{url}            
\usepackage{booktabs}       
\usepackage{amsfonts}       
\usepackage{nicefrac}       
\usepackage{microtype}      
\usepackage{adjustbox}

\usepackage{graphicx} 
\usepackage{microtype}
\usepackage{subfig} 
\usepackage{booktabs} 

\usepackage[round]{natbib}
\usepackage{algorithm}
\usepackage{algorithmic}
\usepackage[utf8]{inputenc}
\usepackage[english]{babel}
 \usepackage{dsfont}
\usepackage{lineno,hyperref}
\usepackage{amssymb}
\usepackage{amsmath}
\usepackage{graphicx}
\usepackage{booktabs}
\usepackage{tikz}
\usepackage{diagbox}
\usetikzlibrary{plotmarks,arrows,calc}
\usetikzlibrary{positioning,fit,calc}
\tikzset{block/.style={draw,thick,text width=2cm,minimum height=1cm,align=center,node distance = 5em},
         line/.style={-latex}
}
\usepackage{pgfplots}

\usepackage{todonotes}

\newcommand{\R}{\mathbb{R}}
\newcommand{\E}{\mathbb{E}}
\newcommand{\set}[1]{\left\{ #1 \right\}}
\newcommand{\beq}{\begin{eqnarray*}}
\newcommand{\eeq}{\end{eqnarray*}}
\newcommand{\beqn}{\begin{eqnarray}}
\newcommand{\eeqn}{\end{eqnarray}}
\newcommand{\x}{\vec x}
\newcommand{\w}{\vec w}

\renewcommand{\vec}[1]{\ensuremath{\text{{\bf\textrm{#1}}}}}

\newcommand{\surL}{\tilde L} 
\newcommand{\norm}[1]{\left\lVert#1\right\rVert}

\begin{document}

\title{Apportioned Margin Approach for Cost Sensitive Large Margin Classifiers}


\author{Lee-Ad Gottlieb          \and
        Eran Kaufman 		\and
        Aryeh Kontorovich 
}


\institute{  
          Lee-Ad Gottlieb  \at
              Ariel University \\
              \email{leead@ariel.ac.il}                 
           \and
Eran Kaufman \at
              Ariel University \\
              \email{erankfmn@gmail.com}           
                  \and
           Aryeh Kontorovich 
	  \at
              Ben-Gurion University  \\
              \email{karyeh@cs.bgu.ac.il}   
}

\date{Received: date / Accepted: date}

\maketitle

\begin{abstract}
We consider the problem of cost sensitive multiclass classification,
where we would like to increase the sensitivity of an important class at the 
expense of a less important one.
We adopt an {\em apportioned margin} framework to address this problem,
which enables an efficient margin shift between classes that share the same boundary.
The decision boundary between all pairs of classes divides the margin
between them in accordance to a given prioritization vector, 
which yields a tighter error bound for the important classes
while also reducing the overall out-of-sample error.
In addition to demonstrating an efficient implementation of our framework,
we derive generalization bounds, 
demonstrate Fisher consistency,
adapt the framework to Mercer's kernel and to neural networks,
and report promising empirical results on all accounts.
\end{abstract}

\section{Introduction}
Cost-sensitive learning \citep{DBLP:conf/ijcai/Elkan01} is a widely studied topic in classification, 
with multiple engineering applications including security surveillance \citep{DBLP:conf/mm/2003}, 
geomatics \citep{DBLP:journals/ml/KubatHM98}, telecommunications \citep{DBLP:journals/datamine/FawcettP97},
medicine \citep{1527706}, bioinformatics \citep{doi:10.1093/protein/gzh061}, signal processing \citep{DBLP:journals/tbe/ShaoSOWL09}, 
and handwritten digit recognition \citep{DBLP:journals/pr/LauerSB07,DBLP:journals/corr/McDonnellTST14}.
In this setting, some labels or classes are more ``important'' than 
others, in the sense that an error on these labels is more costly than on the others.
The total cost is the sum over all classes of the probability of erring
on that class, times the importance (or weight) assigned to that class.
A widely used approach for this problem is to assign each point a weight,
equal to the weight of its class. 
%
As pointed out by \cite{IMV-19}, the limitations of this approach are 
apparent even in the simple case of linearly separable data -- 
that is, in the absence of classification errors -- 
where the decision boundary will be placed in the middle of the sets,
irrespective of the different label costs.

More generally, one can show that the probability of erring on a specific class is 
inversely proportional to the distance of that class to the linear classifier 
(see for example Section \ref{sec:stat} and Corollary \ref{cor:generalization}).
This directly implies that the overall cost may be minimized by shifting the
margin away from the important class (Figure \ref{fig:bin_border} (c)), 
and further that the optimal shift 
is determined by the proportion of the weights of the two classes.
Motivated by these statistical considerations, and 
in contradistinction to point cost-based solutions, 
we consider a multiclass classification
approach based on apportioned margin. Here, the decision boundary
between adjacent classes is shifted away from the more important class towards the less important
class, based on the statistically optimal proportion.
This has the effect of increasing the margin of one class at the cost of reducing the margin from
another. Thus, the out-of-sample error probability for the important class is reduced,
as is the overall cost.

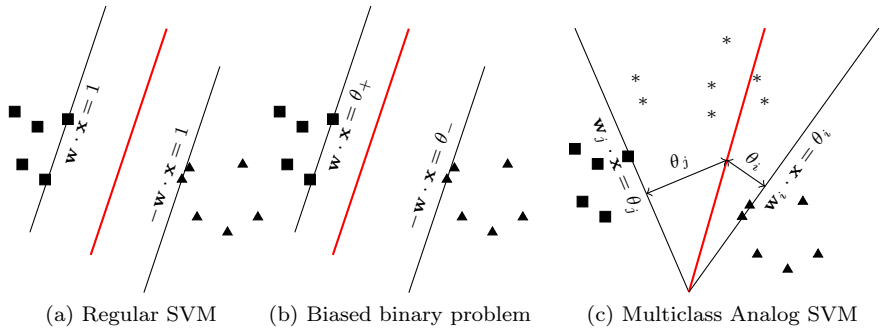
\begin{figure}    

\subfloat[Regular SVM]{\label{fig:a}{

\begin{tikzpicture}
\scriptsize
  \draw[red,thick] (6.5,-1.5) -- node [sloped, anchor=center, above, text width=2.0cm]
  {} (7.5,1.5);

\draw  (6.7,1.8) -- node [sloped,below] {$\vec w \cdot \vec x = 1$} +($(7,-1)-(8,2)$);
\draw (8.2,1) -- node [sloped,above] {$- \vec w \cdot \vec x =1$} +($(7,-1)-(8,2)$);

  \foreach \Point in {(5.9,-0.5), (5.8,0.2), (6.2,0.3), (5.5,0.4), (5.6,-0.3)}{
    \node at \Point {\pgfuseplotmark{square*}};};

\foreach \Point in {(8.7,-1), (7.9,-1), (8.5,-0.3), (8.3,-1.2),(7.7,-0.5),(7.8,-0.35)}{
\node at \Point  {\pgfuseplotmark{triangle*}};};
\end{tikzpicture}}}
\subfloat[Biased binary problem]{\label{fig:b}{
\begin{tikzpicture}
\scriptsize
  \draw[red,thick] (6.2,-1.5) -- node [sloped, anchor=center, above, text width=2.0cm]
  {} (7.2,1.5);
  
\draw  (6.7,1.8) -- node [sloped,below] {$\vec w \cdot \vec x = \theta_+$} +($(7,-1)-(8,2)$);
\draw (8.2,1) -- node [sloped,above] {$ - \vec w \cdot \vec x = \theta_-$} +($(7,-1)-(8,2)$);

  \foreach \Point in {(5.9,-0.5), (5.8,0.2), (6.2,0.3), (5.5,0.4), (5.6,-0.3)}{
    \node at \Point {\pgfuseplotmark{square*}};};

\foreach \Point in {(8.7,-1), (7.9,-1), (8.5,-0.3), (8.3,-1.2),(7.7,-0.5),(7.8,-0.35)}{
\node at \Point  {\pgfuseplotmark{triangle*}};};  
\end{tikzpicture}
}}\quad
  \subfloat[Multiclass Analog SVM]{\label{fig:c}{

\begin{tikzpicture}
\scriptsize
  \draw (7,-1.5) -- node [sloped,below] {$\vec w_j \cdot \vec x = \theta_j$} (5.5,2);
  \draw[red,thick] (7,-1.5) -- node [sloped, anchor=center, above, text width=2.0cm]
  {} (8,2);
  \draw (7,-1.5) -- node [sloped,below] {$\vec w_i\cdot\vec x =  \theta_i $} (9.5,2);

  \foreach \Point in {(5.9,-0.5), (5.8,0.2), (6.2,0.3), (5.5,0.4), (5.6,-0.3)}{
    \node at \Point {\pgfuseplotmark{square*}};};

  \foreach \Point in {(6.3,1.3), (6.4,1) ,(7.3,0.82), (7.5,1.8), (8,1), (7.9,1.3), (7.3,1.2)}{
\node at \Point  {*};};

\foreach \Point in {(8.7,-1), (7.9,-1), (8.5,-0.3), (8.3,-1.2),(7.7,-0.5),(7.8,-0.35)}{
\node at \Point  {\pgfuseplotmark{triangle*}};};

 \draw [<->,above] (7.5,0.25) -- node [sloped]
  {$\theta_i$}($(7,-1.5)!(7.5,0.25)!(9.5,2)$) ;
  
   \draw [<->,above] (7.5,0.25) -- node [sloped]
  {$\theta_j$}($(7,-1.5)!(7.5,0.25)!(5.5,2)$) ;
\end{tikzpicture}
}}
\caption{Support hyperplanes and their bisector.}
\label{fig:bin_border}
\end{figure}

In this paper, we present our apportioned margin framework, explain its
advantage over previous approaches, and show how to efficiently implement 
its associated algorithms (Sections \ref{sec:framework},\ref{SGD}).
We prove that our new framework has strong statistical foundation 
(Sections \ref{sec:stat}),
and present promising experiments on real-world data (Section \ref{sec:exp}).

\subsection{Background and related work}\label{sec:background}
\paragraph{Binary cost-sensitive classification.}
In binary classification there are two primary methods for inducing cost on a  classification:
The first is by changing the bias term \citep{DBLP:journals/pr/Bradley97,DBLP:journals/tkde/HuangL05}.
 In this framework we find a balance classifier $h(\vec x)$ between the two classes, then we create 
 a new classifier $h'(\vec x) = h(\vec x) - \theta$. By modifying the values of $\theta$ we can cause the sensitivity
 of one class to grow at the expense of the other, i.e. to prefer false positives over false negatives.
 
 
A second common approach in binary classification is the class weighting approach.
The weighted support vector machine (WSVM) was originally
proposed by \citet{DBLP:journals/tnn/LinW02} 
and further developed by \citet{10.1007/978-3-642-25661-5_47, DBLP:journals/ijon/AnL13,ke2013}.
These algorithms assign weights to data examples based on their importance.
Here, the cost coefficients are directly factored into the SVM optimization problem \citep{DBLP:journals/tnn/LinW02}: 
\begin {equation}
\begin{aligned}
 &\min &&\frac{\lambda}{2}||W||^2+C^+\sum_{i | y_i = +1} \xi_i +  C^-\sum_{j | y_j = -1} \xi_{j}\\
&s.t. &&\vec w \cdot \vec x_i +b \geq   1- \xi_i\\
&&&\xi_{i} \geq 0 \quad \forall i
\end{aligned}
\end{equation}
where $C^+$ and $C^-$ are the different costs of the two classes.
A different formulation
assigns cost to points  instead of classes \citep{Yang2005WeightedSV}:
\begin {equation}
\begin{aligned}
 &\min &&\frac{\lambda}{2}||W||^2+\sum s_i \xi_i \\
&s.t. &&\vec w \cdot \vec x_i +b \geq   1- \xi_i\\
&&&\xi_{i} \geq 0 \quad \forall i
\end{aligned}
\end{equation}
where $s_i$ is the weight 
of the $i$th example.

The individual weights can be either be chosen via inverse class size
for an unbalanced set 
\citep{JMLR:v17:14-526,DBLP:journals/ml/FungM05, DBLP:books/sp/FernandezGGPKH18}
or
via
kernel-based probabilistic $c$-means (KPCM)
for outlier detection \citep{Yang2005WeightedSV}.

\paragraph{Multiclass SVM.}
There are two major approaches to multiclass SVM classification.
The first approach decomposes the $k$-class problem into multiple binary
classification subproblems:
The problem can be decomposed into $k$ one-vs-all
or $
{k\choose2}
$ one-vs-one binary problems
(respectively, max-win and all pairs),
and the solutions combined by majority vote.
The second approach is to solve the multiclass problem directly by
incorporating the multiple classes into a single optimization model, see
\citet{DBLP:conf/esann/WestonW99, DBLP:journals/coap/BredensteinerB99, DBLP:journals/jmlr/CrammerS01}.
These approaches do not incorporate cost-sensitivity into their main objective function.
Rather, they
impose
a single objective function 
for training $k$ binary SVMs simultaneously while maximizing
the margins from each class to all remaining ones.
  Given a labeled training sample of size $n$ represented by $\{(\vec x_{1},y_{1}), \ldots ,(\vec x_{n},y_{n})\}$, 
  where $\vec x_i \in \mathbb{R}^{d}$, $y \in \{ 1, \ldots ,k \} $, define the $k\times d$
  matrix $W$ as consisting of row vectors
  $\vec w_j$
  corresponding to the hyperplane separating class $j$ from the rest.
\citet{DBLP:conf/esann/WestonW99} formulate the optimization problem as:
\begin {equation}
\begin{aligned}
 &\min &&\frac{\lambda}{2}||W||_2^2+\sum_{i=1}^{n} \sum_{j \neq y_j} \xi_{ij}\\
&s.t. && \vec w_{y_i} \cdot \vec x_i +b_{y_i} \geq   \vec w_{y_j} \cdot \vec x_i +b_{y_j}+2- \xi_{ij} \quad \forall i,j \\
&&&\xi_{ij} \geq 0  \quad \forall i,j,
\end{aligned}
\end{equation}
where $||W||_2$ is the Frobenius
norm of the matrix $W$, 
and serves as a regularizer to prevent overfitting.
\citet{DBLP:journals/jmlr/CrammerS01} proposed :
\begin{equation}
\begin{aligned}
 &\min &&\frac{\lambda}{2}||W||_2^2+\sum_{i=1}^{n}  \xi_{i} \\
&s.t. && \vec w_{y_i} \cdot \vec x_i  -  \vec w_{y_j} \cdot \vec x_i  \geq   1-\delta_{y_i,j}-\xi_{i} \quad \forall i,j \\
&&& \xi_{i} \geq 0 \quad \forall i,
\end{aligned}
\end{equation}
where $\delta_{y_i,j}$ is the Kronecker delta. 

\paragraph{Weighted multiclass.}
Let a {\em priority vector} be of the form $\theta \in (0,\infty)^k$,
and this assigns different costs to different labels. 
One can construct weighted versions of the above multiclass algorithms,
namely cost-sensitive one-vs-one (CSOVO), 
cost-sensitive ove-vs-all (CSOVA), 
cost-sensitive Crammer-Singer (CSCS), etc.\ \citep{DBLP:conf/kdd/JanWLL12}.
Other methods suggest using a cost matrix $C_{y,y'}$ where there is not just a single cost associated 
with misclassifying a class,
but rather different costs are applied for misclassification of one class to different classes.
In this category 
  \citet{doi:10.1198/016214504000000098} suggested altering the multiclass formulation 
 by estimating the conditional label distribution $P(Y = j | X = x)$,
 and employing the Bayes-optimal classifier $\text{argmin}_y \mathbb{E}_{P(y|x)} C_{y,y'}$, 
where $C_{y,y'}$ is the element of the cost matrix.
This was further investigated by \cite{doi:10.1198/jcgs.2010.09206}
via a
reinforced multicategorial approach.
\citet{JMLR:v17:11-229} suggested a unified view of the multiclass support vector machines covering most variants.
If we adopt the empirical risk minimization framework, 
we can define a decision function vector $f = (f_1 (x), \ldots, f_k (x))^T \in  \mathbb{R}^k$, 
where each component corresponds to one class.
Then the prediction rule is $\hat{y} = \text{argmax}_{j\in \{1,\ldots,k\}}f_j(x)$ for any new data point x,
and the optimization formulation can typically be written as
$\min_{f \in F_j} \lambda \mathbb{J}(f) + \sum_{i=1}^n \ell (f(x_i),y_i)$.
The first part of this objective function is the penalty term  $\mathbb{J}$ to prevent overfitting, 
the second part is the empirical loss term, 
and $\lambda$ is a tuning parameter that balances the loss and penalty terms.
In Table \ref{tab:sum} we use this terminology to summarize the different approaches.
(See also \citet{DBLP:conf/uai/AsifXBZ15} for an adversary constrained zero-sum game approach,
and \citet{10.1093/biomet/asu017,DBLP:journals/ma/FuZL18} for an angle-based large-margin classifier.)

\begin{table*}
 \begin{tabular}{|ll|}
    \toprule
    	{\small {\bf Loss function} $ \ell (f(x),y)$ } & {\small \bf Method} \\
    \midrule
	{ $[1-f_y(x)]_+ $}  &{one-vs-all}\\
    \midrule
	{ $\sum_{j \neq y} [1-(f_y(x)-f_j(x))]_+ $} &{ \citet{DBLP:conf/esann/WestonW99} }   \\
    \midrule
	{ $[1- \min_{ j \neq y}(f_y(x)-f_j(x))]_+ $} &{ \citet{DBLP:journals/jmlr/CrammerS01} }   \\
    \midrule
	{ $\sum_{j \neq y} [1+f_j(x)]_+ $} &{ \citet{doi:10.1198/016214504000000098} }   \\
    \midrule
	{ $\gamma [(k-1) - f_y(x)]_+  + (1-\gamma) \sum_{j \neq y} [1+f_j(x)]_ +$} &{ \citet{doi:10.1198/jcgs.2010.09206} }   \\
    \bottomrule
\end{tabular}%
\caption{Various loss functions. $[.]_+$ is the ramp function $[u] = \max(u,0)$.}
\label{tab:sum}
\end{table*}

\paragraph{Our contribution.}
The above large-margin classifier approaches are based on point 
misclassification. In contrast, we suggest directly imposing a
margin proportional to the importance of the weight of a class
-- in effect, improving the 
sensitivity of one class at the expense of another.
As one can prove that the error bound of a class
is inversely proportional to its margin (Section \ref{sec:stat}),
this approach has sound theoretical foundation.
That other approaches do not directly impose a proportional margin
is clearly illustrated by the simple case of linearly separable data,
as shown in Figure \ref{fig2} in the experimental section). 
In this case, the above methods all fail to
shift the margin away from the more important classes in proportion to the 
cost vector, while our method is precisely tailored for this purpose.
The above methods can indeed move the margin in response to a misclassification,
but this does not give the desired ratio which we directly impose.\footnote{It is
perhaps conceivable that these methods can shift the margin to the statistically
justified proportion via an ad-hoc use of the regularization parameter, but his would require 
the introduction and cross-validation of at least $k$ separate parameters, which is an intensive task.}
It is therefore unsurprising that our method consistently out-performs the 
others experimentally (Section \ref{sec:exp}).

%
%

\begin{figure}%
    \centering
\includegraphics[width=8cm]{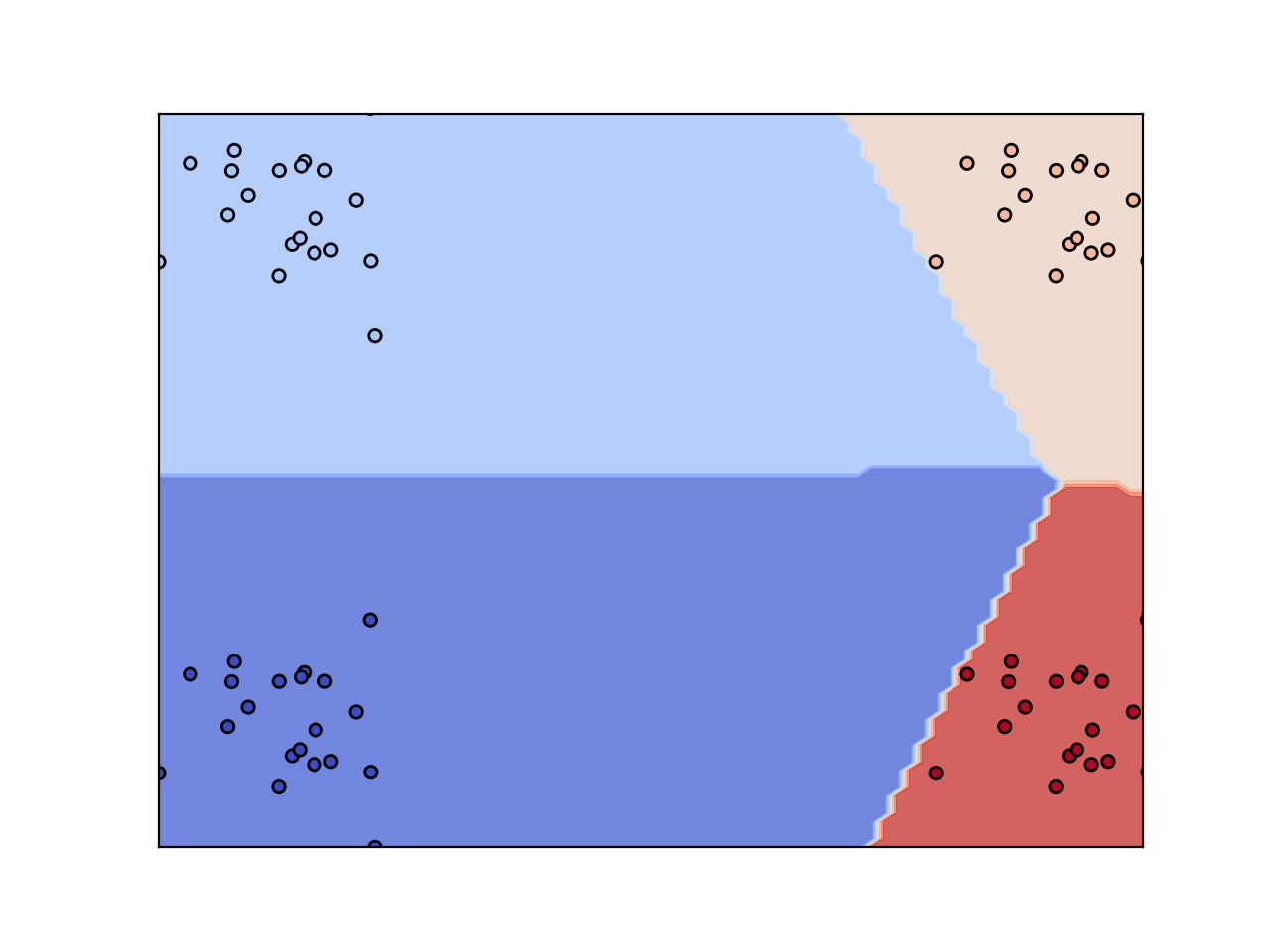}

\caption{An example of a desired decision boundary with 
priority vector $\{10,10,1,1\}$.}
\label{fig:border1}
 \end{figure}

\paragraph{Fisher consistency.}
Define $P_j(x)=P(Y=j | X=x)$, and a classifier with loss $\ell (f (x), y)$ is 
{\em Fisher consistent} if the minimizer of $E_P[\ell (f (x), y)]$ 
has the property $\text{argmax}_j f_j = \text{argmax}_j P_j$.
Although the binary SVM is known to be Fisher consistent, not all MSVMs are.
 In particular, in Table \ref{tab:sum} rows $1-3$ are known to not always be consistent.
In contrast, row 5 is always Fisher consistent when 
$\gamma \in [0, 0.5]$, which also covers row 4 as a special case with $\gamma = 0$ \citep{pmlr-v2-liu07b}.
Recall that the cost function of the cost-sensitive classification problem 
is the sum over all classes of the probability of erring on that class
times the weight assigned to that class;
we will define a cost-sensitive classifier to be Fisher consistent if the minimizer of $E_P[\ell (f (x), y)]$ 
has the property $\text{argmax}_j f_j = \text{argmax}_j \theta_j P_j$ . 
In Section \ref{sec:stat} we prove that the our algorithm is cost-sensitive Fisher consistent.


\section{The apportioned margin framework}\label{sec:framework}
Given a cost sensitive problem, a desired property for 
our classifier is to impose larger margins for the more ``expensive'' labels.
%
Our cost-sensitive framework allows for a flexible method for shifting
the decision boundary between different pairs of classes.
Intuitively, we ``budget'' the regions between conflicting classes according to some priority vector.
This goal is
illustrated graphically in Figure \ref{fig:border1}:
The two classes on the left have identical costs, 
but their cost is greater than that of the two 
classes on the right, which also have identical costs.
As a result, the horizontal boundaries
are centered, but the vertical boundaries are shifted to the right.
As shown in the experimental section Figure \ref{fig2}, the widely used methods for
cost sensitive multi-classification do not achieve this.
%
%
%
%
%
%
%
%

In binary SVM, the solution vector 
$\vec{w}$ defines a separator, whose margin depends on 
$\norm{\vec{w}}$. At the two edges of the margin lie the 
hyperplanes $\vec{w} \cdot \vec{x} \pm 1 = 0$.
If we were to denote $\vec w = \vec w_+ = - \vec w_-$ as two classifiers for the positive and negative  
examples  respectively,
then the binary hyperplanes can be reformulated as $\vec w_j x  = 1$ and the decision function $\hat{y} = \text{argmax}_j \vec w_j x$. 
(See Figure \ref{fig:bin_border} (a)).
Conversely if we were to scale the samples with $\theta_j$ and shift the decision boundary by a ratio 
of $\frac{\theta_+}{\theta_-}$ (which is also the ratio between their Bayesian probabilities),  
then the hyperplanes could be formulated as $\vec w_j x = \theta_j $ 
and the decision function $\text{sign} (\vec w  x)$ 
written as $\hat{y} = \text{argmax}_j \frac{\vec w_j x}{\theta_j}$.
Fortunately, this formulation can be extended to multiclass categorization.


By analogy to the binary hard-margin setting, 
we assume that each hyperplane separates its class from all others.
While in the binary SVM settings the two hyperplanes are parallel, 
in the more general multiclass problem the different hyperplanes 
can intersect (see Figure \ref{fig:bin_border} (c)).
Consider two hyperplanes $\vec{w}_i, \vec{w}_j$ separating classes $i,j$.
The decision boundary between them is a bisecting hyperplane.
We desire that the ratio $\frac{\theta_{i}}{\theta_{j}}$ define the distance from 
a weighted bisector to the hyperplanes, meaning that the ratio of the scaled distance of the 
weighted bisector from class $i$ to its distance from class $j$ 
should be $\frac{\theta_{i}}{\theta_{j}}$.
(A weighted bisector is illustrated in Figure \ref{fig:bin_border}.)
The following lemma formalizes the geometric intuition of the interaction
between two neighboring classes.
\begin{definition}
Let 
$w^{ij} = w_i - w_j$, 
$\overline w_j = \frac{\vec w_j}{\theta_j}$ 
and 
$\bar w^{ij} = \bar w_i - \bar w_j$.
\end{definition}
\begin{lemma}\label{lem:place}
The set of all points whose
ratio of scaled distances from the hyperplanes 
$\vec w_i \cdot \vec x = \theta_i$ and $\vec w_j \cdot \vec x = \theta_j$
is $\frac{\theta_i}{\theta_j}$,
is given by the hyperplane $ \overline {w}_{ij} \vec x = 0$.
\end{lemma}


\begin{proof}
$\frac{\vec w_j \cdot \vec x - \theta_j}{\vec w_i \cdot \vec x -\theta_i} = \frac{\theta_j}{\theta_i}
\implies \overline{w}_j \cdot \vec{x} - 1 = \overline{w}_i \cdot \vec{x} -1
\implies (\overline w_j-\overline w_i) \vec x = \overline {w}^{ij} \vec x = 0.$
\end{proof}

Lemma~\ref{lem:place} gives the decision rule between two classes $i,j$, and implies that our
overall decision function is
 \begin{equation}
 \hat y = \text{argmax}_{l \in [k]} \overline w_l \cdot \vec x
   .
  \end{equation}
  
\subsection{The Optimization Formulation}

Taking $\theta_y$ to be the cost of making a mistake on label $y$
implies the following loss on the example-label pair $(\vec x , y )$:
%
%
\begin{equation}
\begin{split}
&L(\vec x , y)= {\begin{cases} 
\theta_y, &{\text{if }}  \max \{\overline w_l \} \cdot \vec x - \overline w_y \cdot \vec x  >  0\\
0 &{\text{otherwise}}.
\end{cases}}
\end{split}
\end{equation}  
As desired, this loss function is asymmetric and discourages error on
relatively ``important''
classes.
Unfortunately this problem subsumes that of minimizing zero-one loss,
which known to be NP hard \citep{Hoffgen:1995:RTS:207270.207282}.
Instead we propose the following convex relaxation:
\begin{equation}
\begin{split}
&\surL_j (\vec x , y)= {\begin{cases} 
\theta_y - \bar \delta_{y,j} \vec w_j \cdot \vec x, &{\text{if }}  \theta_y \geq \bar \delta_{y,j} \vec w_j \cdot \vec x \\
0 &{\text{otherwise }}
\end{cases}}\\
  &\surL(\vec x , y)= \sum_j \surL_j(\vec x , y)
  ,
\end{split}
\end{equation}  
where $\bar\delta_{y_i,j}$ is the signed Kronecker delta:
\begin {equation}
 \bar\delta_{y,j} ={\begin{cases}+1,&{\text{if }} \quad y = j\\-1,&{\text{otherwise. }}\\\end{cases}}
\end {equation}


This relaxation is
a multiclass analogue of the hinge loss, with a zero penalty above
a certain margin threshold and a linearly increasing penalty below it.


It is easy to see that $\surL(\vec x ,y) \geq L(\vec x,y)$. Note that in this formulation of the cost, an example belonging to class $j$ is 
charged not only for a misclassification by its own classifier $j$, but is also charged when a different classifier $i$ is not within a scaled 
distance of $\theta_i + \theta_j$ of the example in $j$. 
This results in a margin shift.

\paragraph{Primary Formulation.}
As suggested in the previous section, in the separable case we want all examples of class $j$
to be above the plane $w_j x + b_j = \theta_j$.
The following optimization problem
is a natural multiclass analogue of hard-margin
maximization:
\begin {equation}
\begin{aligned}
&\min &&||W||^2 \\
&s.t. && \bar\delta_{y_i,j}(\vec w_j \cdot \vec x_i+b_{j}) \geq \theta_{y_i}
\end{aligned}
\end{equation}

Indeed, $||W||^2$ bounds the sum of the pairwise margins:
\begin{lemma}
\label{sumLab}
$\sum_{s} \sum_{r < s} \norm{\vec w^{rs}}^2 \leq k \norm{W}^2$.
\end{lemma}
\begin{proof}
Let us define $\sum_{r=1}^{k} \vec w_r = \vec c$. Then
$\norm{\vec c}^2= (\sum_{r=1}^{k} \vec w_r)(\sum_{r=1}^{k} \vec w_r) 
= \sum_{r=1}^{k} \norm{\vec w_r}^2 + 2\sum_s \sum_{r<s} \vec w_r \cdot \vec w_s$.
%
%
Using this we obtain:
\begin{equation}
\label{eq1}
\begin{aligned}
\sum_{s} \sum_{r < s} \norm{\vec w_r - \vec w_s }^2 
&= (k-1) \sum_{r=1}^{k} \norm{\vec w_r}^2 -2\sum_s \sum_{r<s} \vec w_r \cdot \vec w_s \\
&= (k-1) \sum_{r=1}^{k} \norm{\vec w_r}^2 - (\norm c^2 -  \sum_{r=1}^{k} \norm{\vec w_r}^2) \\
&= k \sum_{r=1}^{k} \norm{\vec w_r}^2 - \norm c^2 \\
&\leq k \sum_{r=1}^{k} \norm{\vec w_r}^2 = k \norm{W}^2.
\end{aligned}
\end{equation}
\end{proof}

%
%


By relaxing the constraints we obtain the soft margin formulation.

\begin {equation}
  \label{eq:soft-margin}
\begin{aligned}
&\min &&\quad \frac{1}{2}||W||^2+C \sum_{ij} \xi_{ij} \\
&s.t. &&\quad  \theta_{y_i}-\bar\delta_{y_i,j}(\vec \vec w_j \cdot \vec \vec x_i+b_{j}) \leq \xi_{ij}  \quad  \forall i,j \\
&     &&\quad  \xi_{ij} \geq 0   \quad   \forall i,j.
\end{aligned}
\end{equation}

\paragraph{Dual formulation.}
The primal formulation in (\ref{eq:soft-margin}) involves searching over a $d$-dimensional space.
A standard transformation to the dual amounts to kernalizing the problem, rendering
the search space dimension-independent.
We begin the the Lagrangian formulation
\begin{equation}
\label{dual}
\begin{aligned}
&\max \quad \mathcal{L} = \frac{1}{2}||W||^2+C \sum_{ij} \xi_{ij} \\
& -\sum_{ij}\alpha_{ij}(\xi_{ij} -( \theta_{j}-\bar\delta_{y_i,j}(\vec \vec w_j \cdot \vec \vec x_i+b_{j}))) \\
& -\sum_{ij}\beta_{ij}\xi_{ij} \\
&s.t. \quad \alpha_{ij},\beta_{ij} \geq 0 \quad   \forall i,j.
\end{aligned}
\end{equation}
Putting
$\overline{\alpha}_{ij}=\alpha_{ij}\bar\delta_{y_i,j}$,
and
and invoking the KKT conditions, we have
\begin {equation}
\label{representer}
\frac{\partial \mathcal{L}}{\partial \vec w_j}=0 \implies \vec w_j=\sum_{i} \overline{\alpha}_{ij}\vec x_i
.
\end{equation}
This is our analogue of the Representer Theorem \citep{DBLP:conf/colt/ScholkopfHS01}.
The second KKT condition is
 \begin {equation}
 \label{dual1}
\frac{\partial \mathcal{L}}{\partial b_{j}}=0 \implies \sum_{ij} \overline{\alpha}_{ij} = 0
\end{equation}

and can also be written as
\begin {equation}
\sum_{i \in j} {\alpha}_{ij} = \sum_{i \notin j} {\alpha}_{ij}. 
\end{equation}

Note that his equation can act as a balancer for an unbalanced sets:
For a particular class $j$ the sum of weights of data belonging to $j$
equals the sum of weights of data not belonging to $j$.

The final KKT condition is: 
$\frac{\partial \mathcal{L}}{\partial \xi_{ij}}=0 
\implies \alpha_{ij} = C - \beta_{ij} 
\implies  0 
\leq \alpha_{ij} 
\leq C$.
Using the $\overline \alpha$ notation we have:
\begin {equation}
  0 \leq \delta_{ij}\overline \alpha_{ij} \leq 1
\end{equation}

Substituting the multiclass Representer Theorem \ref{representer} back into the main equation (\ref{dual}), we derive:
\begin{equation}
\begin{aligned}
&\max \quad  \mathcal{L}=\sum_{j=0}^{k}\sum_{i=0}^{n} \alpha_{ij} \theta_{y_i} - \frac{1}{2} \norm W^2
=\sum_{j=0}^{k}(\sum_{i=0}^{n} \alpha_{ij} \theta_{y_i} - \sum_{i=0}^{n}\sum_{l=0}^{n} \overline{\alpha}_{ij} \overline{\alpha}_{lj} \vec{x}_i\cdot\vec{x}_l) \\
&s.t. \quad  \quad  \sum_{ij} \overline{\alpha}_{ij} = 0.
\end{aligned}	
\end{equation}

It follows from the KKT conditions that:
\begin {equation}
\label{dual3}
\overline{\alpha}_{ij}(\theta_{y_i}-(\vec \vec w_j \cdot \vec \vec x_i+b_{j})) = 0
.
\end{equation}
When $\overline{\alpha}_{ij} = 0$ we have that $\theta_{y_i}-(\vec \vec w_j \cdot \vec \vec x_i+b_{j}) \neq 0 $, 
while if $\overline{\alpha}_{ij} \neq 0$ we have that $\theta_{y_i}-(\vec \vec w_j \cdot \vec \vec x_i+b_{j}) = 0 $.
The vectors $\vec x_i$ where $\overline{\alpha}_{ij} \neq 0 $ 
can be considered to be the support vector for class $j$.
These points lie on the hyperplane $\vec \vec w_j \cdot \vec x +b_{j} = \theta_j$,
which is the ``support hyperplane'' for class $j$.

\section{Statistical error bounds}\label{sec:stat}
In this section we will present generalization bounds.
We first obtain tight bounds in the realizable case
via the decision directed acyclic graph (DDAG) approach,
and then present the Rademacher framework
for the general (agnostic) case.

\paragraph{Directed Graph Approach.}
\citet{Platt1999LargeMD} 
considered a classifier for the multiclass problem, 
which takes the form of a binary directed acyclic graph (DAG);
they termed their classifier a decision DAG (DDAG).
Each class is represented as a terminal node in the graph,
and given a test point, the algorithm begins at root of the graph
and traverses down towards a terminal. At each node, two
classes are considered, and each exiting path cannot reach
one of the classes, meaning that the decision made at the node
will effectively rule out one class from consideration.
It follows that the DDAG has depth $k-1$. 
%
It is immediate that the DDAG can be used to model our decision function
with depth $k-1$ and $k-1$ nodes for each class, where
we use $\bar{w}^{ij}$ to decide between pairs of classes.
The generalization bounds presented in \cite{Platt1999LargeMD} for a given label is:

\begin{theorem}[\cite{Platt1999LargeMD}]\label{thm:ddag}
Suppose we are able to correctly distinguish class $i$ from all other classes in a random $n$ sample of labeled examples using a SVM DDAG on $k$ classes
  with margin $\gamma_{ij}$ at node ${i,j}$,
 then we can bound the generalization error on class $i$ with probability greater than $1-\delta$ to be less than
\begin{equation}
\epsilon(i) \leq \frac{130R^2}{n}(D \log{4en}\log{4n}+\log{\frac{2(2n)^{k-1}}{\delta}}),
\end{equation}
where $R$ is the radius of a ball containing the support of the distribution and $D=\sum_j \frac{1}{\gamma_{ij}^2}$. 
\end{theorem}

As an immediate corollary of Theorem \ref{thm:ddag} we have:

\begin{corollary}\label{cor:generalization}
Suppose we are able to correctly distinguish between all $k$ classes in a random $n$ sample of
cost sensitive labeled examples using a SVM DDAG 
with margin $\gamma_{ij}$ at node ${i,j}$.
Then we can bound the total loss to be less than
$
\sum_i \theta_i \epsilon(i)
$.
\end{corollary}

In both Theorem \ref{thm:ddag} and its corollary,
$\gamma_{ij}$ is the margin between class $i$ and the separating hyperplane between class $i$ and class $j$:
$\gamma_{ij} = \min_{x \in i} \frac{(\overline w_i - \overline w_j)\vec x}{\norm{\overline w^{i,j}}}$.
Let $\eta_{ij} = \frac{\theta_i}{\theta_j}$, and we can bound $\gamma_{ij}$ as follows:


\begin{lemma}
\label{margin}
The margin $\gamma_{ij} = \min_{x \in i } \frac{(\overline w_i - \overline w_j)\vec x}{\norm{\overline w^{i,j}}}$ statisfies
$\gamma_{ij} \geq \frac{1 + \eta_{ij}}{\norm{\overline w^{ij}}}$.
\label{lem:mar}
\end{lemma}

\begin{proof}
By definition,
$\gamma_{ij} = \min_{x \in i } \frac{(\overline w_i - \overline w_j)\vec x}{\norm{\overline w^{i,j}}}.$
From the constraints we have both that
$\theta_{y} - \vec w_y \cdot \vec x \leq 0 \implies \vec w_y \cdot \vec x \geq \theta_{y}$
and 
$\theta_{y} + \vec w_j \cdot \vec x \leq 0 \implies -\vec w_j \cdot \vec x \geq \theta_{y}$,
$\forall y \neq j.$
It follows that
$(\overline w_i - \overline w_j) \vec x \geq \frac{\theta_i}{\theta_i} +  \frac{\theta_i}{\theta_j} = 1 + \eta_{ij}$,
and we conclude that
$\gamma_{ij} \geq \frac{1 + \eta_{ij}}{\norm{\overline w^{ij}}}$.
\end{proof}
Note that when $\theta_i \gg \theta_j$, the error on the $i$-th class becomes small, as desired. 

\paragraph{Rademacher complexity.}
Let $f_j$ be a function mapping $(\x,y)$ to $\R$
via $f_j(\x,y) = [\theta_y\pm\w_j\cdot\x]_+$,
where $\theta_y=\theta(y)$ is a fixed function $\theta:[k]\to\R_+$
mapping the labels to positive reals, $\x,\w_j\in\R^d$
have Euclidean norm at most $R$ and $\Lambda_j$, respectively.
Let us bound the Rademacher complexity
of the function class $F_j$ of all such $f_j$ indexed by the $\w_j$,
restricted to the given range of $\x$:
\begin{equation}\label{eq:gen}
\begin{aligned}
R_n(F_j) & =
\E\sup_{\w_j}\frac1n\sum_{i=1}^n\sigma_i[\theta(y_i)\pm\w_j\cdot\x_i]_+ \\
&\le
\E\sup_{\w_j}\frac1n\sum_{i=1}^n\sigma_i(\theta(y_i)\pm\w_j\cdot\x_i) \\
&=
\E
\frac1n\sum_{i=1}^n\sigma_i\theta(y)
+
\E
\sup_{\w_j}\frac1n\sum_{i=1}^n\sigma_i\w_j\cdot\x_i \\
&\le
\frac1n\sum_{i=1}^n\theta(y)\E\sigma_i
+
\E
\sup_{\w_j}\frac1n\sum_{i=1}^n\sigma_i\w_j\cdot\x_i\\
&=
\E
\sup_{\w_j}\frac1n\sum_{i=1}^n\sigma_i\w_j\cdot\x_i
\le
\frac{R\Lambda_j}{\sqrt n},
\end{aligned}
\end{equation}
where the first inequality follows from the Talagrand contraction lemma
\citep[Lemma 4.2]{mohri-book2012}, \citep[Theorem 4.12]{LedouxTal91},
and the second from standard Rademacher estimates on linear classes
\citep[Theorem 4.3]{mohri-book2012}.

The error margin for class $j$ is proportional to $\Lambda_j$ which is 
inverse proportional to margin of class $j$ and as proven in the previous section
 our method sets the margins to be proportional to the $\theta$'s
 and hence the more costly examples have tighter bounds, which is the desired effect.
 
%
%
%

Finally, consider the function classes $F_j$, $j\in[k]$,
each parametrized by $\w_j$ with Euclidean norm at most $\Lambda_j$.
Define $F$ as the (Minkowski) sum of these classes:
\beq
F=\set{ (\x,y)\mapsto\sum_{j=1}^k f_j(\x,y) : f_j\in F_j}
\eeq
and recall that Rademacher complexities are sub-additive.
Then
\beq
R_n(F)
\le
\frac{R\sum_{j=1}^k\Lambda_j}{\sqrt n}
\le
R||W||\sqrt{\frac{k}{n}}.
\eeq
This explains why in the optimization formulation we minimized $||W||_2^2$.


Let us clip our loss at $\theta_{\max}$,
so $\tilde L':=\min\set{\tilde L,\theta_{\max}}$.
Then, by 
\cite[Theorem 3.1]{mohri-book2012},
we have
that with probability $\ge1-\delta$,
\beq
\E_{(X,Y)}[\tilde L'(X,Y)]
 \le \frac1n\sum_{i=1}^n\tilde L'(X_i,Y_i)
+ 2R_n(\min\set{\theta_{\max},F})
+ \theta_{\max}\sqrt{\frac{\log(1/\delta)}{2n}}.
\eeq
Since truncation by $\theta_{\max}$ is a $1$-Lipschitz transformation,
the Talagrand contraction lemma implies that
$R_n(\min\set{\theta_{\max},F})\le R_n(F)$,
if we also normalize the loss function w.l.o.g. by $\theta_{\max}$ we obtain our final bound,
\beq
\E_{(X,Y)}[\tilde L'(X,Y)]
\le \frac1n\sum_{i=1}^n\tilde L'(X_i,Y_i)
+ 2R \frac{||W||}{\theta_{\max}}\sqrt{\frac{k}{n}}
+ \sqrt{\frac{\log(1/\delta)}{2n}},
\eeq
which holds with probability $\ge1-\delta$.

\subsection{Fisher Consistency}
Given the set of functions $F_j$ we will say that a classifier with loss $\ell (f (x), y)$ is 
Fisher consistent if the minimizer of $E_P[\ell (f (x), y)]$ 
has the property $\text{argmax}_j f_j = \text{argmax}_j \theta_j P_j$.
We note that this ensures the intuitive property that for point $(x,y)$ we have
$\frac{P_y(x)}{P_j(x)} \geq \frac{\theta_j}{\theta_y} \quad \forall j \neq y$. 
In order to prove the our classifier is Fisher consistent, without loss of generality we
impose the additional constraint that $\sum f_j(x) = 0$.

\begin{lemma}
The minimizer $\bold f^*$ of  $E_p[ \sum_{j = 1}^{k} [\theta_y - \overline \delta_{y,j} f_j ]_+]$ subject to
$\sum_{j = 1}^{k}  f_j = 0$ satisfies the following: 
$f^*_j (x) = \theta_j$ if $j= \text{argmax}_j \theta_j P_j(x)$ 
and $- \frac{\theta_{j}}{k-1}$ otherwise.
\end{lemma}

\begin{proof}
By defintion,
$E[ \sum_{j = 1}^{k} [\theta_y - \overline \delta_{y,j} f_j ]_+] = \sum_l P_l \sum_j  [\theta_l - \overline \delta_{l,j} f_j ]_+$.
We first show that the minimizer $\vec f^*$ for point $(\vec x ,y)$ satisfies $\overline \delta_{y,j}f_j^* \leq \theta_y \quad \forall y,j$.
Suppose by way of contradiction that the optimal solution  $\vec f^*$ 
has $\overline \delta_{y,j} f_j > \theta_y$ for some $j$;
then we can construct another solution $ \vec f'$ with
 $f_j' = \overline \delta_{y,j} \theta _y$ and 
$f_l'= f_l^* +  \frac{f_j' - \overline \delta_{y,j} \theta _y}{k-1} \quad \forall l \neq j$.
The second solution satisfies
$ \overline \delta_{y,j}  f_j' \geq  \overline \delta_{y,j} f_j^*$ 
while still satisfying the constraint $\sum f_j^* = \sum f_j'$, 
However the objective function $ \sum_{j = 1}^{k} [\theta_y - \overline \delta_{y,j} f'_j ]_+ \leq \sum_{j = 1}^{k} [\theta_y - \overline \delta_{y,j} f^*_j ]_+ $,
which is a contradiction.

Using this property of $f^*$, we can rewrite the objective function as
$\sum_l P_l \sum_j  \theta_l - \overline \delta_{l,j} f_j $. 
Since $\sum f_j = 0$ implies $\sum \overline \delta_{y,j}  f_j = 2 f_j$,
and then the objective function can be written as:
$\sum_l P_l \sum_j  \theta_l - \overline \delta_{l,j} f_j =  \sum_l P_l \sum_j  \theta_l - 2 f_j $
which is equivalent to:
\beq
&\text{max}_{F_j} \sum P_l f_l \\
&s.t. \quad \sum f_j = 0 \quad  f_j \overline \delta_{jl} \leq \theta_j \quad \forall j
\eeq
If we define $\hat y = \text{argmax}_j \theta_j P_j$, 
it is easy to see that the solution satisfies $f^*_{\hat y} = \theta_{\hat y} $  and $- \frac{\theta_{\hat y}}{k-1}$ otherwise.
\end{proof}

\section{SGD Based Solver}\label{SGD}
In this section, we present a Stochastic Gradient Descent (SGD) based learning algorithm
 for  solving our algorithm. 
 Although SGD based algorithms  are not the optimized solution for convex problems, 
 the are widely used in non convex problems such as NeuralNets. 
Initially, we set each $\vec w_j$ to the zero vector,
 At iteration $t$ of the algorithm,
  we choose a random training example $(\vec x_{t}, y_{t})$
   by picking an index uniformly at random and compute the sub-gradient
\begin {equation}
\nabla_{t,j} \mathcal{L} = \lambda \vec w_{j,t} -\mathds{1}[\theta_{y_t}-\bar\delta_{y_t,j} \vec w_{j,t} \vec x_{t} ]
 \bar\delta_{y_t,j} \vec w_{j,t} \vec x_{t},
\end{equation}
where $\mathds{1}[.]$ is the indicator function (which takes a value of 1 if its
argument is positive and zero otherwise).

The update rule is $\vec w_{j,t+1} \leftarrow \vec w_{j,t} - \eta_t \nabla_{t,j}$,
 where $\eta_t$ is the step size at iteration t.
  Following the Shalev-Singer Pegasos implementation \citep{DBLP:journals/mp/Shalev-ShwartzSSC11}
   we take $\eta_t = \frac{1}{\lambda_t} \quad \lambda_t = \lambda \cdot  t$, and then the update rule can therefore be rewritten as: 
\begin {equation}
\vec w_{j,t+1}  \leftarrow \underbrace{(1-\frac{1}{t})\vec w_{j,t}}_{*}
- \underbrace{\frac{1}{\lambda_t} \mathds{1}[\theta_{y_t}-\bar\delta_{y_t,j} \vec w_{j,t} \vec x_{t}  ] 
\bar\delta_{ y_t, j} \vec x_{t}}_{**}
.
\end{equation}

Here (*) can be viewed as a momentum term, smoothing through past results,
  and (**) applies a $\frac{1}{t}$ weight decay over the gradient of the loss function.
The algorithm achieves $\epsilon$-accuracy 
in time $\mathcal{O}(\frac{1}{\lambda \epsilon})$. 
Note, that unlike other cost sensitive implementations, an example belonging to class $j$ does not 
only affect $\vec w_j$ but also affects all other $\vec w$'s as well by forcing them to retain a scaled margin from
that example.

\paragraph{Dual problem.} 
This result can be generelized to apply Mercer's kernel.
The multiclass Representer Theorem (\ref{representer}) gives:
\begin {equation}
\vec w_j=\sum_{i=0}^{n} \overline \alpha_{i,j}  \vec x_i
\end{equation}
This implies that the optimal solution can be expressed as a linear combination of the training instances,
making it possible to train and utilize an SVM without direct access to the training instances.
In this case, solving the dual problem is not necessary since we can  directly minimize the primal
while still using kernels:
\begin {equation}
\label{update}
\vec w_{j,t+1}=\frac{1}{\lambda_t} \sum_{i=0}^t \mathds{1}[\theta_{y_t}-\bar\delta_{y_t,j} \vec w_{j,t} \cdot \phi(x_{t}) ]
\bar\delta_{y_t,j} \phi(x_{t})
.
\end{equation}
We can now combine equations (\ref{update}) and (\ref{representer}) to get the following lemma:
\begin{lemma}
  The weight at stage $t+1$ is given by:
\begin {equation}
\vec w_{j,t+1}=\frac{1}{\lambda_t} \sum_{i=0}^n \overline \alpha_{i,j,t+1} \phi(x_{i}),
\end{equation}
where $\overline \alpha_{i, j,t+1}$ is given by:  
\begin{equation}
\label{update1}
\overline \alpha_{i, j,t} + 
{\begin{cases} 
 \delta_{i,x_t} \bar\delta_{y_t, j} , &{\text{if }}  \theta_{y_t} > \frac{\bar\delta_{y_t,j}}{\lambda_t} 
\sum_{i=0}^n \overline \alpha_{i ,j,t} K(\vec x_i,\vec x_{t}) \\
0 &{\text{otherwise. }}
\end{cases}}
\end{equation} 
\end{lemma}

\begin{proof}.
We prove the equation by induction.
Suppose the update rule holds for $t-1$, then: 
\begin{enumerate}
\item When the indicator function is zero the update rule is:
\begin {equation}
\label{update2}
\begin{aligned}
&\vec w_{j,t+1}=\frac{t-1}{t} \vec w_{j,t} = \frac{t-1}{t} \frac{1}{\lambda_{t-1}} \sum_{i=0}^n \overline \alpha_{i,j,t} \phi(\vec x_i)\\
& = \frac{1}{\lambda_t} \sum_{i=0}^t \overline \alpha_{i,j,t} \phi(\vec x_i)
\end{aligned}
\end{equation}

\item When the indicator function is non-zero then:
\begin {equation}
\label{update3}
\begin{aligned}
&\vec w_{j,t+1} = \frac{1}{\lambda_t} \sum_{i=0}^n \overline \alpha_{i,j,t} \phi(\vec x_i) +
 \frac{1}{\lambda_t} \bar\delta_{y_t,j} \phi(x_{t}) \\
&= \frac{1}{\lambda_t} \sum_{i=0}^n (\overline \alpha_{i,j,t}+\delta_{i,x_t}\bar\delta_{y_t,j}) \phi(\vec x)
\end{aligned}
\end{equation}
\end{enumerate}

Combining equations (\ref{update1}) and (\ref{update3}) we proved the update rule.
\end{proof}

The final decision function is:
\begin{equation}
\hat y = \text{argmax}_j \frac{1}{\theta_j}\sum_{i=0}^n \overline \alpha_{i,j} K(\vec x_i,x)
\end{equation}
This is a simple and elegant update rule which only depends on the vector $\theta$ and $K(,)$.

\section{Experiments}\label{sec:exp}

For our experiments, we utilized the python scikit-learn library \citep{scikit-learn} 
and our own SGD implementation \footnote{https://github.com/erankfmn/Apportioned-margin-classifiers}.
We used SVM RBF kernel, the regularization parameter $C$ was 5-fold cross-validated 
over the set $\{2^{-5}, 2^{-3}, \ldots, 2^{15}\}$, 
 and the gamma
parameter was 5-fold cross-validated over the set $\{2^{-15}, 2^{-3},\ldots, 2^{3}\}$. 

\subsection{Illustrative example}
Before presenting the experiments, we give an example that illustrates the 
power of our approach. We generated 2-dimensional data
from four separable classes with different biases, and generated their data points 
using a normal distribution.
Figure \ref{fig:exp} shows a comparison of the results of different cost vectors.
This example illustrates how our approach shifts the decision boundary away from high-cost
classes, while still maintaining the convexity of the solution space.
%

\begin{figure}[h]
\centering
\includegraphics[width=10cm]{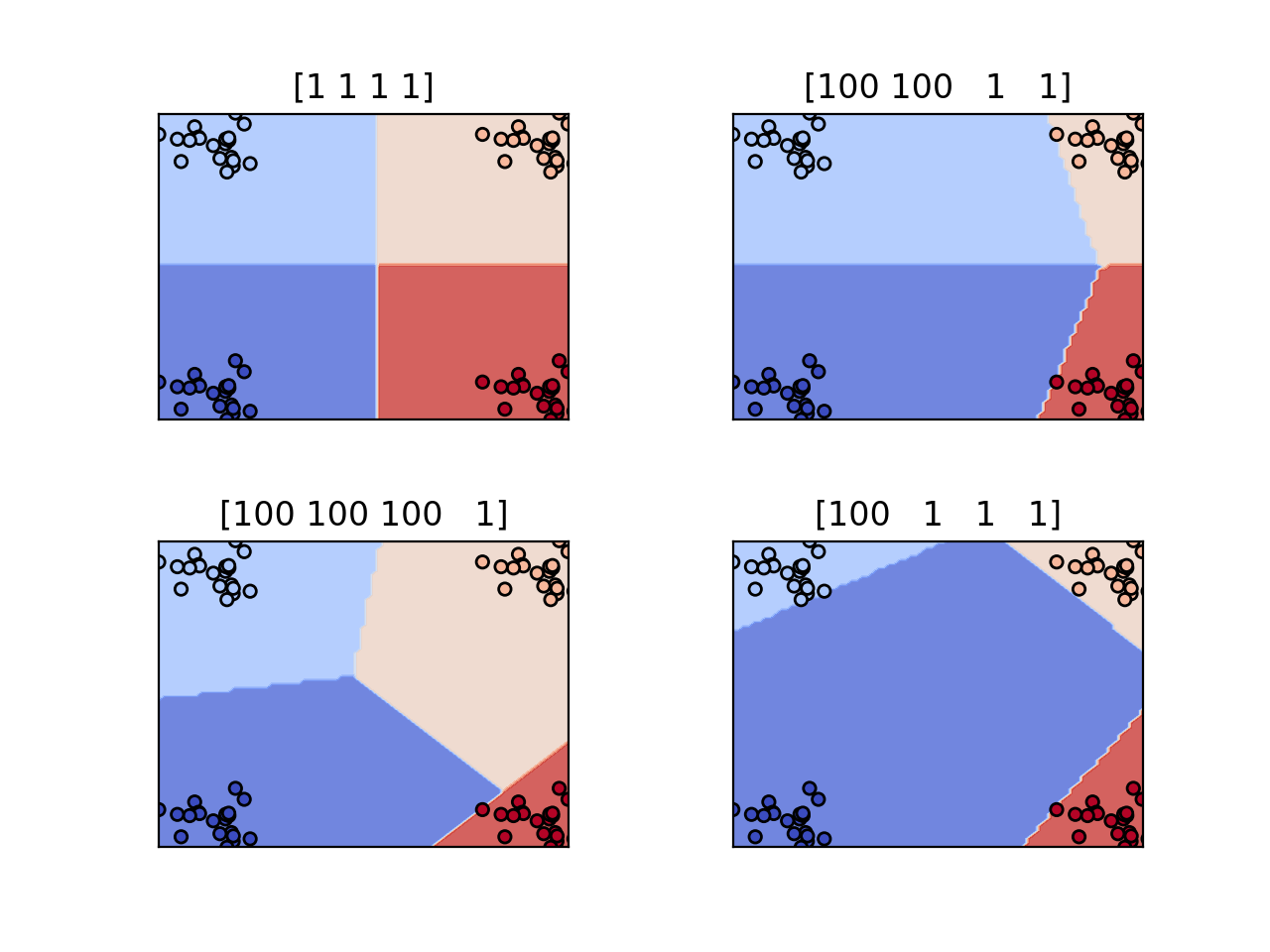}
\caption{Separating classes with different prioritization vectors
assigned to classes \{blue, light blue, yellow, red\}.}
\label{fig:exp}
\end{figure}

\begin{figure}[h]
\centering
\includegraphics[width=8cm]{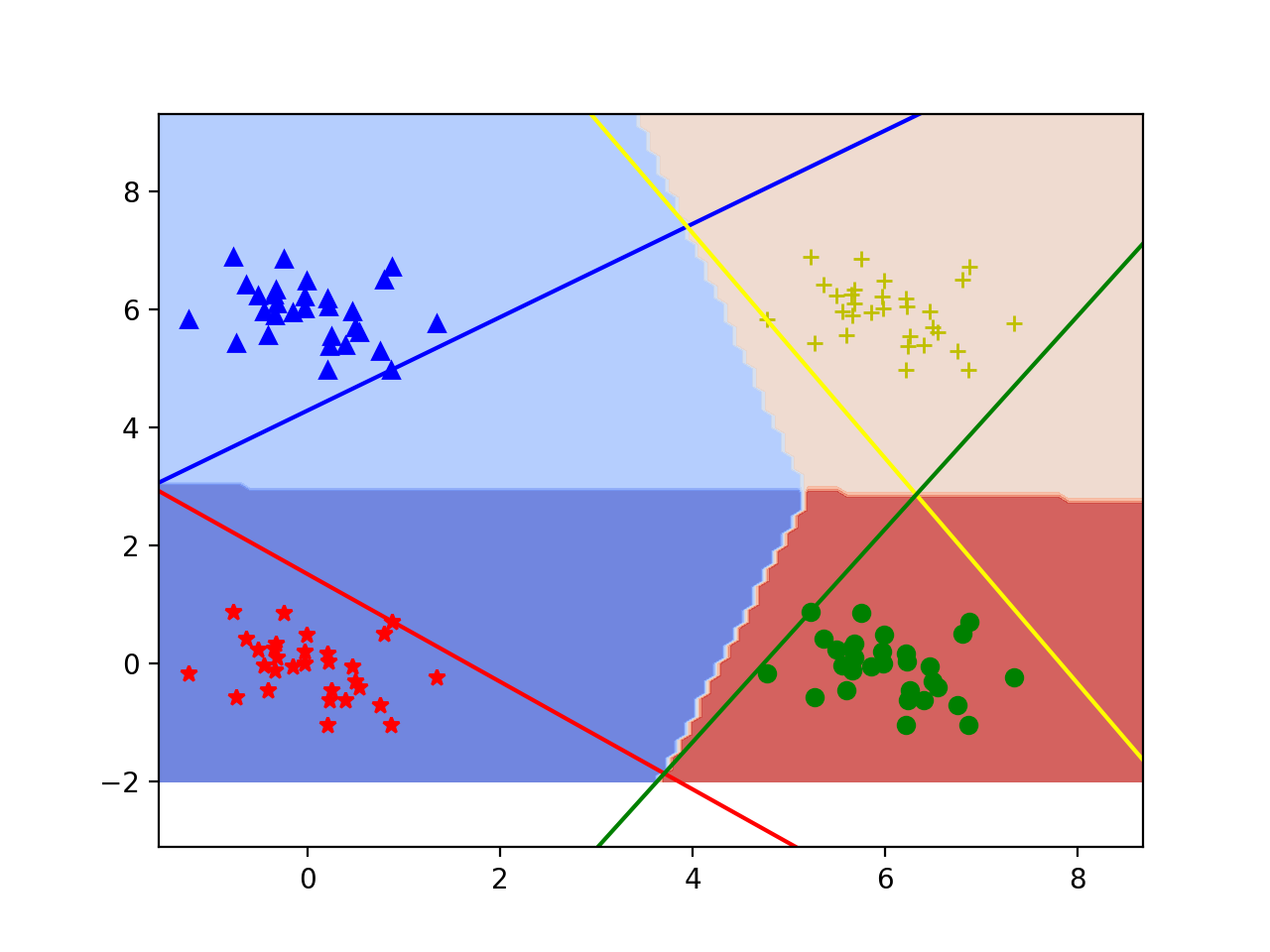}
\caption{Decision boundary and the  $w_j x + b_j = \theta_j$ hyperplanes.}
\end{figure}

%

We also compared our result against the CSOVO, CSOVA and CSCS methods described in Section \ref{sec:background}.
Figure \ref{fig2} shows a comparison of these different methods for the same cost vector.
It is evident that the other methods were unable to significantly shift the decision boundary
away from the high-cost classes.

\begin{figure}[h]
\centering
\includegraphics[width=12cm]{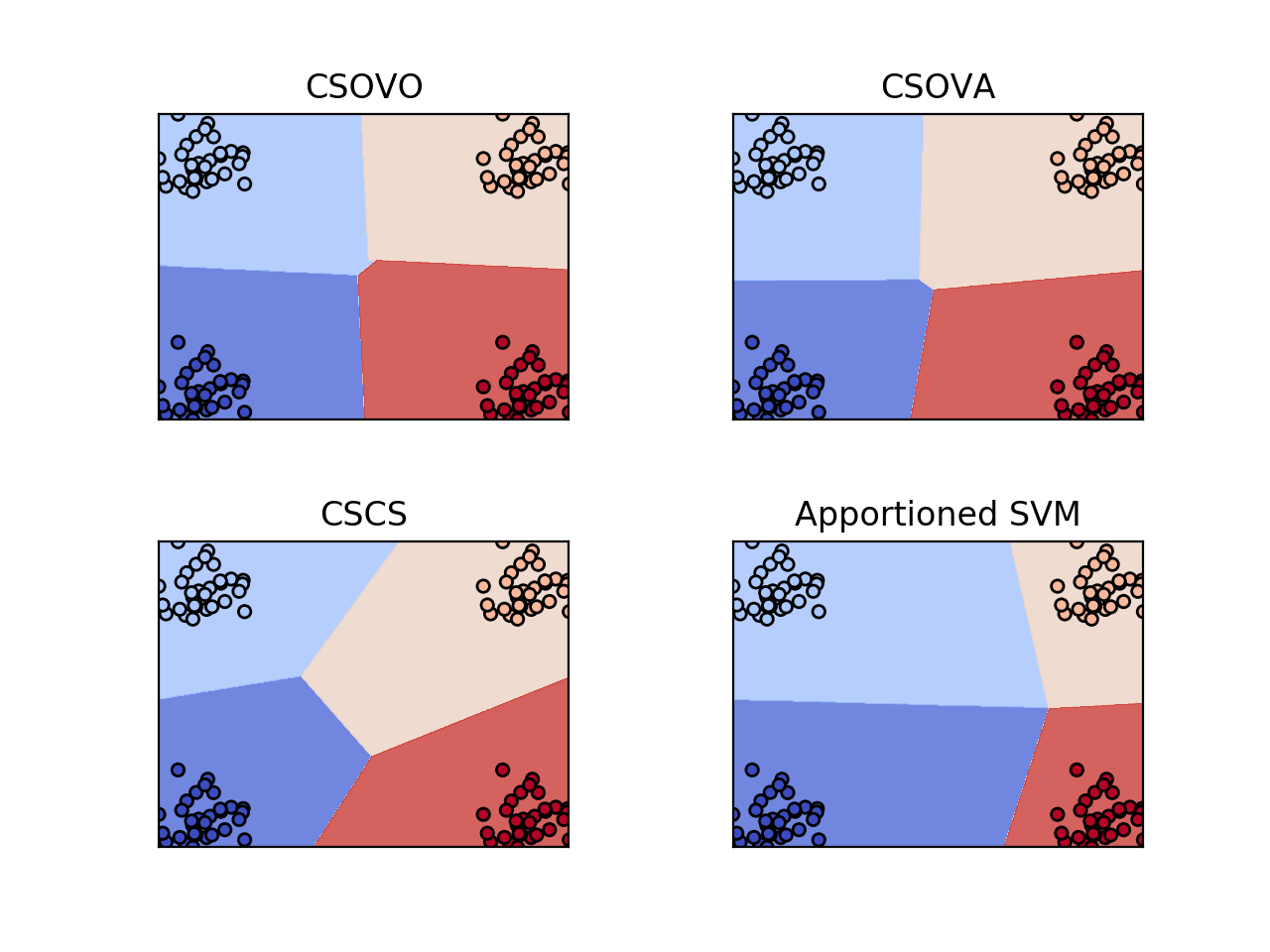}
\caption{A comparison of decision boundaries between different algorithms for
priority vector $\{10,10,1,1\}$.}
\label{fig2}
\end{figure}

\subsection{Benchmarks}
We considered various benchmark datasets from  
LibSVM Machine Learning repository \citep{CC01a}.
Three of these (Heart , Breast, Diabetes) are 
healthcare-oriented datasets where the prioritization is inherent 
(false positives are preferred over false negatives).
We set the cost vector for the more important
class to be twice than the rest for all datasets except for German Credit (Statlog),
where the  cost
vector of $\{5,1\}$ was already stipulated by the authors. 
We also considered several multicategorial datasets,
and randomly chose one of the classes to be more important than the others.
The datasets characteristics are summarized in Table \ref{tab:dat}  

The reported score is expected risk (the sum of the number of misclassification 
for each class times its cost, normalized by the data size).
Results were averaged over 10-fold
cross validation.
We compared our method against other weighted multiclass methods, 
and the results are presented in Table \ref{tab:uci}.
Our method compares favorably with the others reducing the cost by $10-20\%$.
We also report the sensitivity of the most important class. 
These show empirically that, as claimed, our method successfully increased 
the sensitivity of the most important class, and did this more 
effectively than all other methods.

\begin{table}
    \centering
\captionof{table}{Dataset characteristics}
 \begin{tabular}{lcccccccccc}
      \toprule
      				                        &{\small \#classes}        & {\small \#examples}            &{\small \#dims}        \\ 
      \midrule
          {\small Breast Cancer}           & {\small 2}                    & {\small 683}                         & {\small 10}    \\
   	  {\small Diabetes}                   & {\small 2}                     & {\small 768}                          & {\small 8}       \\
	   {\small Heart}                       & {\small 2}                     & {\small 270}                          & {\small 13}       \\
           {\small German}                    & {\small 2}                    &{\small 1000}                        & {\small 20}  \\
           {\small Glass}                        &{\small 6}                    & {\small 214}                               & {\small 9} \\
            {\small Iris}                        &{\small 3}                    & {\small 150}                               & {\small 4} \\
            {\small Vehicle}                    & {\small 4}                     & {\small 846}                         & {\small 18}       \\
            {\small Letter}                    & {\small 26}                     & {\small 20000}                         & {\small 16}       \\
 \end{tabular}       
 \label{tab:dat}  
\end{table}

\begin{table}
\begin{center}
 \begin{tabular}{llcccc}
      \toprule
      	   {\small Dataset}  & {\small our SVM} & {\small CSOVA} & {\small CSCS}  & {\small CSOVO} \\

          \midrule
          {\small Breast}   & {\small 0.058} & {\textbf {0.054}} & {\small 0.056}& {\small 0.054} \\
         			 & {\textbf {86\%}} & {\small 84\%} & {\small 83\%}& {\small 84\%} \\
       \midrule
          {\small Diabetes}    & {\textbf  {0.343}}    & {\small 0.353} & {\small 0.346}  & {\small 0.351} \\
                   			 & {\textbf {77\%}} & {\small 74\%} & {\small 74\%}& {\small 74\%} \\
	 \midrule
          {\small Heart}    & {\small  {0.215}}    & {\small 0.218} & {\textbf {0.214}}  & {\small 0.218} \\
                   		  & {\textbf {80\%}} & {\small 81\%} & {\small 80\%}& {\small 81\%} \\
          
       \midrule
          {\small German}        & {\textbf  {0.250}}      & {\small 0.300}     & {\small 0.3000}    & {\small 0.300} \\
           				  & {\textbf {74\%}} & {\small 70\%} & {\small 70\%}& {\small 70\%} \\

     \midrule
          {\small Iris}      & {\textbf  {0.033}}           & {\small 0.053}          & {\small 0.043}        & {\small 0.053} \\
				     & {\textbf {97\%}}          & {\small 95\%}            & {\small 93\%}         & {\small 95\%} \\
  \midrule
          {\small Glass}      & {\textbf  {0.491}}           & {\small 0.501}          & {\small 0.495}        & {\small 0.501} \\
				     & {\textbf {74\%}}          & {\small 63\%}            & {\small 65\%}         & {\small 63\%} \\

	\midrule
          {\small Vehicle}   & {\textbf  {0.262}}           & {\small 0.281}          & {\small 0.272}         & {\small 0.281}\\
                                     & {\textbf {76\%}}          & {\small 64\%}         & {\small 62\%}          & {\small 60\%} \\
          \midrule
          {\small Letter}   & {\textbf  {0.067}}           & {\small 0.231}          & {\small 0.173}         & {\small 0.231}\\
 			         & {\textbf {96\%}}          & {\small 91\%}         & {\small 91\%}          & {\small 91\%} \\

            \bottomrule
    \end{tabular}
\caption{Benchmarks for cost-sensitive classification}
\label{tab:uci}
\end{center}
\end{table}

\subsection{Cost Sensitive NeuralNets}
The results of the linear classifiers encourage us to adapt the framework to neural networks (NNs),
as cost-sensitive NNs are not well-understood.
A naive approach could simply multiply the output of the final loss layer 
(usually a softmax layer) by the appropriate weights, 
but this leads to poor performance (Table \ref{tab:SVM}).
\citet{DBLP:conf/ecai/KukarK98} suggested the cost function 
$\sum_{j=0}^{k}\sum_{i=0}^{n} ((y_j - o_j) C(i,j))^2$
where $o_j$ is the actual output of the $j$-th output neuron, 
$y_j$ is the desired output,
and $C(i,j)$ is the cost for misclassifying example $i$ as $j$.

We executed the following experiment on the use of NNs to identify superclasses.
We used the CIFAR-100 dataset, a dataset of small images each 
labelled with a class, where all classes are themselves grouped into superclasses.
For our experiment, we chose two superclasses and two subclasses for each superclass.
These were superclass {\em aquatic mammals} with subclasses dolphin and beaver,
and superclass {\em flower} with subclasses orchid and sunflower.
Our prioritization vector was $\{1,1,2,2\}$,
assigning double priority to aquatic mammals.
Each subclass had 500 training images and 100 testing images.
Our premise was that that our cost function could be used to improve image recognition
via classes, that is to train a classifier to favor misclassifying a dolphin as a beaver over
a dolphin as an orchid.

For a more detailed description of the NeuralNet architecture:
The architecture of our net is summarized in Figure \ref{Network Architecture}.
The net is similar to the one introduced by \cite{Krizhevsky:2017} also known as ``AlexNet'':
It contains three convolutional layers and two fully connected layers,
each followed by a batch normalization (BN) 
and dropout layers with probability $0.75$.
We used ReLU as our activation function.
The first layer CONV-1 uses a kernel of $5 \times 5 \times 24$ with stride $1$ .
The second layer CONV-2 uses a kernel of $5 \times 5 \times 48$ with stride $2$.
The third layer CONV-3 uses a kernel of $4 \times 4 \times 64$ with stride $2$.
The fourth layer is a 200 neurons fully connected layer (FC) followed by a hinge loss (HL).
In the final layer we compared three different possible tools,  
the regular cross entropy with softmax and class weighting, SVM hinge-loss with class weighting and our prioritization loss.
We trained our models using Adam gradient descent \citep{Kingma14} 
with a batch size of 100 examples and an exponential weight decay.

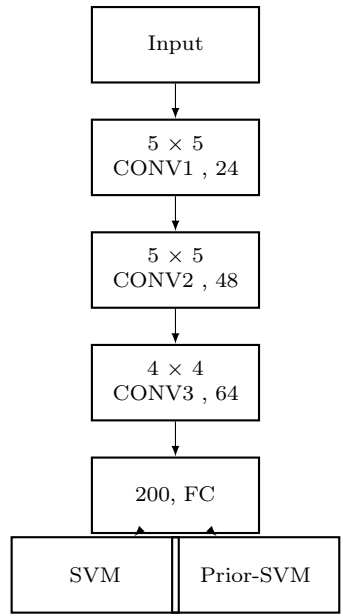
\begin{figure}
\begin{center}

\begin{tikzpicture}
  \node[block] (a) {Input};
  \node[block, below of = a] (b) {$5 \times 5$ CONV1 , 24};
  \draw[line] (a)-- (b); 
   
  
  \node[block, below of = b] (d) {$5 \times 5$ CONV2 , 48};
  \draw[line] (b)-- (d); 
   
  
  \node[block, below of = d] (f) {$4 \times 4$ CONV3 , 64};
  \draw[line] (d)-- (f); 
   
  
  \node[block, below of = f] (h) {200, FC};
  \draw[line] (f)-- (h);

  \node[block, below left of = h] (k) {SVM};
 \node[block, below right of = h] (j) {Prior-SVM};

  \draw[line] (h)-- (k);
 \draw[line] (h)-- (j);

\end{tikzpicture}
\end{center}

\caption{Network Architecture}
\label{Network Architecture}

\end{figure}

Our final result are presented in Table \ref{tab:SVM}, which shows the
overall cost function, and also the sensitivity of the preferred superclass
(aquatic mammals), for NNs utilizing the different algorithms (Weighted softmax , Weighted Hinge loss , Apportioned margin loss).

\begin{table}
\begin{center}
\begin{tabular}{ l l l}
	\toprule
	{\small Method}	& {\small Sensitivity}    &{\small Cost} \\
	\midrule
	{\small Weighted softmax}	& 81.5\%    & 0.186 \\  
	{\small Weighted Hinge loss}	& 80.5\%     & 0.201 \\  
	{\small Apportioned margin loss}  & 86\%     & 0.153 \\ 
	\bottomrule
\end{tabular}
\caption{Superclass experiment}
\label{tab:SVM}
\end{center}
\end{table}

\section{Conclusions}
We introduced the apportioned margin framework which places the cost on the margins rather than on misclassification.
This framework guarantees  an tighter out-of-sample error bound for more important classes
sometime at the expense of less important classes according to a user-defined priority vector.
We presented both  linear, kernelized, and NeuralNet 
vesrsions
 for this framework and 
demonstrated the success of our method on different datasets.
%

\newpage
\bibliographystyle{spbasic}      
\bibliography{mybibfile}   

\end{document}